\documentclass{article}

\usepackage{microtype}
\usepackage{graphicx}
\usepackage{subcaption}
\usepackage{IEEEtrantools}
\usepackage{booktabs} % for professional tables

\usepackage{hyperref}

\usepackage{xcolor}
\usepackage{tcolorbox}
\tcbuselibrary{skins, breakable}
\usepackage{pifont}
\newcommand{\xmark}{\ding{55}}

\usepackage{multirow}

\usepackage[preprint]{icml2026}

\usepackage{amsmath}
\usepackage{amssymb}
\usepackage{mathtools}
\usepackage{amsthm}
\usepackage{color, colortbl}
\usepackage{makecell}

\usepackage[capitalize,noabbrev]{cleveref}

\theoremstyle{plain}
\newtheorem{theorem}{Theorem}[section]
\newtheorem{proposition}[theorem]{Proposition}
\newtheorem{lemma}[theorem]{Lemma}

\theoremstyle{definition}

\theoremstyle{remark}

\DeclareMathOperator*{\argmin}{arg\,min}

\renewcommand{\algorithmicrequire}{\textbf{Input:}}
\renewcommand{\algorithmicensure}{\textbf{Output:}}

\definecolor{Highlight}{rgb}{0.89,0.89,0.94}

\usepackage[textsize=tiny]{todonotes}

\icmltitlerunning{ Stabilizing LLM Supervised Fine-Tuning via Explicit Distributional Control}

\begin{document}

\twocolumn[
  \icmltitle{ Stabilizing LLM Supervised Fine-Tuning via Explicit Distributional Control}

  % % It is OKAY to include author information, even for blind submissions: the
  % % style file will automatically remove it for you unless you've provided
  % % the [accepted] option to the icml2026 package.

  % % List of affiliations: The first argument should be a (short) identifier you
  % % will use later to specify author affiliations Academic affiliations
  % % should list Department, University, City, Region, Country Industry
  % % affiliations should list Company, City, Region, Country

  % % You can specify symbols, otherwise they are numbered in order. Ideally, you
  % % should not use this facility. Affiliations will be numbered in order of
  % % appearance and this is the preferred way.
  \icmlsetsymbol{equal}{*}

  \begin{icmlauthorlist}
    \icmlauthor{Xinyu Wang}{equal,sch}
    \icmlauthor{Changzhi Sun}{equal,comp}
    \icmlauthor{Yuanbin Wu}{sch}
    \icmlauthor{Xiaoling Wang}{sch}
  \end{icmlauthorlist}

  \icmlaffiliation{comp}{TeleAI, Shanghai, China}
  \icmlaffiliation{sch}{School of Computer Science and Technology, East China Normal University, Shanghai, China}

  \icmlcorrespondingauthor{Xiaoling Wang}{xlwang@cs.ecnu.edu.cn}

  % \icmlkeywords{Machine Learning}
  % \printAffiliations

]

% this must go after the closing bracket ] following \twocolumn[ ...

% This command actually creates the footnote in the first column listing the
% affiliations and the copyright notice. The command takes one argument, which
% is text to display at the start of the footnote. The \icmlEqualContribution
% command is standard text for equal contribution. Remove it (just {}) if you
% do not need this facility.

% Use ONE of the following lines. DO NOT remove the command.
% If you have no special notice, KEEP empty braces:
% \printAffiliationsAndNotice{}  % no special notice (required even if empty)
% Or, if applicable, use the standard equal contribution text:
% \printAffiliationsAndNotice{\icmlEqualContribution}
\printAffiliationsAndNotice{}

\section{abstract}
Post-training large language models (LLMs) often suffers from catastrophic forgetting, where improvements on a target objective degrade previously acquired capabilities. Recent evidence suggests that this phenomenon is primarily driven by excessive distributional drift during optimization.
Motivated by this perspective, we propose Anchored Learning, a simple framework that explicitly controls distributional updates during offline fine-tuning via a dynamically evolving \emph{moving anchor}.
Instead of matching a fixed reference distribution, the anchor interpolates between the current model and a frozen reference to construct an intermediate target that the model distills toward, transforming global fine-tuning into a sequence of local trust-region updates in distribution space.
Theoretically, we prove this anchor-based update admits a linear KL-divergence upper bound per iteration, ensuring a stable transition between model distributions.
Extensive experiments on iGSM, MedCalc, and IFEval show that Anchored Learning consistently lies on the Pareto frontier of gain–stability trade-offs, achieving near-optimal performance improvements while substantially reducing degradation compared to strong baselines.
For example, while standard SFT suffers from over 53\% performance degradation on iGSM and MedCalc, Anchored Learning slashes this drop to under 5\% while maintaining near-optimal gains (e.g., 75.2\% on iGSM).
% \begin{abstract}
% \input{000abstract}
% \end{abstract}

\section{Introduction}
\label{sec:intro}
Post-training large language models (LLMs) to adapt them to new objectives is a central step in modern model development, enabling improved task performance, alignment, and domain specialization\cite{labrak-etal-2024-biomistral,peng2024ecellm}. 
However, post-training often induces \emph{catastrophic forgetting}, where gains on the target task come at the expense of degraded performance on previously acquired capabilities~\cite{kirkpatrick2017overcoming,zhai2023investigating}. 
Understanding the mechanism behind this phenomenon is critical for designing stable and efficient fine-tuning algorithms.

Recent empirical and theoretical studies converge on a common conclusion: catastrophic forgetting is primarily driven by excessive \emph{distributional drift} during optimization. 
\emph{RL’s Razor} shows that although reinforcement learning (RL) and supervised fine-tuning (SFT) often achieve comparable target-task performance, RL preserves prior capabilities significantly better, with the degree of forgetting strongly correlated with the KL-divergence between the fine-tuned and base policy, reflecting an implicit bias of on-policy RL toward KL-minimal solutions among all policies that solve the objective~\cite{shenfeld2025rl}.
Complementarily, the robustness of RL has been attributed to the mode-seeking behavior induced by on-policy data rather than to algorithmic components such as KL regularization or advantage estimation, with approximately on-policy updates already substantially mitigating forgetting~\cite{chen2025retaining}. 
Together, these findings indicate that post-training stability is governed by how tightly the optimization trajectory controls distributional shift and mode preservation.

These insights, however, remain largely descriptive: while on-policy RL implicitly constrains distributional drift, it does so through costly data collection and complex optimization pipelines, and offers limited direct control over the optimization trajectory.
In contrast, standard SFT lacks any explicit mechanism for regulating how far and in which direction the model is allowed to move in distribution space, often leading to abrupt behavioral shifts~\cite{lin2025sft,pareja2025unveiling}.
Existing approaches based on static KL regularization, trust-region constraints, or fixed teacher distillation provide only indirect or globally tuned control~\cite{yang2024self,wu2025mitigating,wu2025generalization}, and are poorly aligned with the local geometry of model distributions (Tab.~\ref{tab:method_comparison}).
Therefore, a natural question is \emph{whether a simple and explicit mechanism can enable stable control of distributional updates in a purely offline setting, combining the efficiency of SFT with the stability of on-policy RL}.

\begin{table*}[t]
\centering
\footnotesize
\setlength{\tabcolsep}{4pt}
\caption{
Comparison of representative post-training paradigms for continual adaptation.
\emph{Dyn.~Target} indicates whether the supervision target is explicitly
constructed to evolve with the current model during optimization.
\emph{Prim.~Superv.} denotes the signal that defines the dominant optimization
objective, rather than the input data source.
\emph{Drift Ctrl.} characterizes how distributional updates are constrained,
either implicitly through sampling or losses, or explicitly via surrogate or
KL-based constraints.
\emph{Ctrl.~Gran.} specifies whether the drift constraint operates locally
(pointwise or per-update) or globally over the parameter or policy space.
\emph{Offline} indicates whether training relies solely on a fixed dataset
without online interaction.
Anchored Learning is the only approach that simultaneously admits an explicitly
constructed dynamic target, locally bounded drift control, and a fully offline
training regime.
% Abbreviations: Dyn.~Target = Dynamic Target, Prim.~Superv. = Primary Supervision,
% Drift Ctrl. = Drift Control, Ctrl.~Gran. = Control Granularity.
}
\begin{tabular}{lccccc}
\toprule
\textbf{Method}
& \textbf{Dyn. Target}
& \textbf{Prim. Superv.}
& \textbf{Drift Ctrl.}
& \textbf{Ctrl. Gran.}
& \textbf{Offline} \\
\midrule
Supervised Fine-Tuning
& No
& Labeled Data
& Implicit (Forward KL)
& Global
& Yes \\

Reward-Driven RL
& No
& Reward Signal
& Implicit (Sampling / Exploration)
& Local
& No \\

Surrogate RL (e.g., PPO)
& No
& Reward Signal
& Explicit (Trust-Region / KL)
& Local
& No \\

KL-Regularized Fine-Tuning
& No
& Fixed Reference
& Explicit (KL)
& Global
& Yes \\

Distillation
& No
& Frozen Teacher
& Implicit (KL)
& Global
& Yes \\

\midrule
Anchored Learning (Ours)
& {Yes} 
& {Interpolated Anchor} 
& {Explicit (Bounded)} 
& {Local} 
& {Yes}  \\
\bottomrule
\end{tabular}
\label{tab:method_comparison}
\end{table*}

In this work, we propose \emph{Anchored Learning}, a simple framework that explicitly controls distributional updates during fine-tuning via a dynamically evolving \emph{moving anchor}. Instead of matching a fixed reference distribution, the anchor is constructed by interpolating between the current model and a frozen reference (e.g., an SFT model), yielding an intermediate target that the model distills toward at each iteration. As the anchor co-evolves with the optimization trajectory, successive updates are constrained to remain local in distribution space, transforming global fine-tuning into a sequence of conservative steps that effectively induce an implicit trust-region behavior while preserving the efficiency and simplicity of offline training.

Anchored Learning admits theoretical guarantees on distributional stability. 
In particular, we show that the KL divergence between the anchored target and the current model admits an explicit per-iteration upper bound under both probability-space and logit-space anchoring, ensuring a stable and controlled transition between successive model distributions.
In summary, our main contributions are:
\begin{itemize}
    \item We propose \emph{Anchored Learning}, a simple and general framework that explicitly enforces local control over distributional updates in offline fine-tuning via a moving anchor.
    \item We provide theoretical guarantees on distributional stability by deriving explicit per-iteration KL-divergence bounds for both probability-space and logit-space anchoring.
    \item We demonstrate empirically that Anchored Learning achieves a strong trade-off between performance gains and stability across multiple benchmarks.
\end{itemize}

\section{Preliminaries}
\label{sec:prelim}

\paragraph{Problem Setting.}
We study  model adaptation in a supervised fine-tuning (SFT) setting.

Let $\mathcal{D}=\{(x_i,y_i)\}_{i=1}^N$ denote a labeled dataset for a new task, and let
$\mathcal{D}_X$ denote the empirical marginal distribution over inputs induced by $\mathcal{D}$.
Let $p_{\mathrm{base}}(\cdot\mid x)$ be a fixed base conditional distribution encoding previously acquired knowledge.
Standard SFT minimizes the empirical negative log-likelihood
$\mathcal{L}_{\mathrm{sft}}(\theta)
= \mathbb{E}_{(x,y)\sim\mathcal{D}}[-\log p_\theta(y \mid x)]$
on $\mathcal{D}$ starting from $p_{\mathrm{base}}$,
yielding an SFT reference model $p_{\mathrm{sft}}(\cdot\mid x)$.
Throughout this work, we assume $p_{\mathrm{sft}}$ is trained once and then kept fixed.

\paragraph{Scope and Practical Constraints.}
In practice, the data used to acquire general capabilities in open-source LLMs is often unknown or unavailable, making replay or joint training on such data infeasible.
Accordingly, we restrict our scope to \emph{purely offline fine-tuning on the target dataset} $\mathcal{D}$, excluding replay-based methods, mixed-dataset training, and parameter isolation, and focus on improving stability through controlled optimization under this realistic constraint.

\paragraph{Limitations of Standard SFT.}
While SFT can effectively fit the target distribution induced by $\mathcal{D}$, it provides no explicit mechanism
to preserve prior behaviors. When the target distribution differs substantially from pretraining, gradients may
systematically overwrite previously acquired representations due to parameter sharing and non-convex optimization,
leading to distributional drift and catastrophic forgetting.

A natural baseline for mitigating forgetting is to regularize a trainable model $p_{\theta}(\cdot\mid x)$ toward the fixed base model $p_{\mathrm{base}}(\cdot\mid x)$ via a KL penalty:
\begin{equation}
\min_{\theta}\;
\mathcal{L}_{\mathrm{sft}}(\theta)
+
\lambda\,\mathbb{E}_{(x, y)\sim \mathcal{D}}
\!\left[
\mathrm{KL}\!\left(
p_{\theta}(\cdot \mid x)
\,\|\, 
p_{\mathrm{base}}(\cdot \mid x)
\right)
\right],
\label{eq:naive_kl}
\end{equation}
where $\lambda>0$ controls the regularization strength.
Although intuitive, this approach often performs poorly when the target distribution deviates substantially from the base model. In practice, the gradients of the task loss and the KL penalty frequently conflict, leading to slow convergence or unstable optimization, while attraction toward a fixed base model becomes increasingly misaligned as adaptation progresses, suppressing beneficial task-specific updates and limiting final performance. 

These limitations motivate a mechanism that (i) reduces persistent gradient conflict, (ii) allows the reference distribution to evolve during optimization, and (iii) enforces local, trust-region-like constraints for stable updates.

\begin{figure}[t]
  \vskip 0.2in
  \begin{center}
    \centerline{\includegraphics[width=\columnwidth]{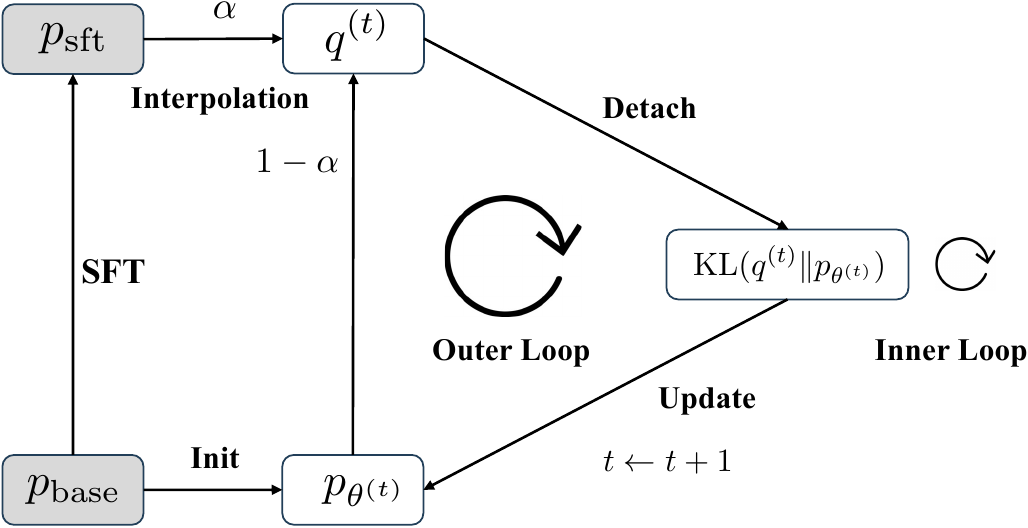}}
\caption{
Overview of the Anchored Learning framework.
Starting from the base model $p_{\mathrm{base}}$, a supervised fine-tuned model
$p_{\mathrm{sft}}$ is obtained via standard SFT and kept fixed thereafter.
At outer iteration $t$, an anchor distribution $q^{(t)}$ is constructed by
interpolating between the current trainable model $p_{\theta^{(t)}}$ and the
frozen SFT model with coefficient $\alpha$.
The anchor is detached from the computation graph and serves as a fixed target
for the inner-loop optimization, which minimizes the distillation objective
$\mathrm{KL}(q^{(t)} \,\|\, p_{\theta})$ to update the model parameters.
The updated model becomes $p_{\theta^{(t+1)}}$ for the next outer iteration.
Gray boxes indicate frozen models ($p_{\mathrm{base}}$ and $p_{\mathrm{sft}}$),
while $p_{\theta^{(t)}}$ is updated during training.
}
    \label{fig:overview}
  \end{center}
\end{figure}

\section{Approach}
\label{sec:approach}

We propose \emph{Anchored Learning}, a two-level framework that explicitly controls distributional drift relative to the base model while enabling effective adaptation to the new dataset $\mathcal{D}$ (Alg.~\ref{alg:anchored_learning}).
Instead of regularizing toward a fixed reference distribution, Anchored Learning constructs a \emph{dynamic anchor} that co-evolves with the model and provides a locally consistent optimization target (Fig.~\ref{fig:overview}).

Anchored Learning decouples \emph{anchor selection} from \emph{anchor fitting}.
In the outer loop, an anchor distribution $q^{(t)}$ specifies the desired
direction of model evolution (Sec.~\ref{sec:outer}), while in the inner loop, the
model is updated to approximate this anchor via a generic distillation objective
(Sec.~\ref{sec:inner}). The anchor is instantiated using interpolation schemes
between the current model and the fixed SFT reference model $p_{\mathrm{sft}}$
(Sec.~\ref{sec:interp}), providing an explicit mechanism to balance stability and
plasticity. This design induces a local trust-region behavior: at each outer
iteration, the model moves only a controlled distance from its current state
toward a task-optimal direction, mitigating persistent gradient conflicts and
excessive drift.

\subsection{Outer Loop: Anchor Selection via Interpolation}
\label{sec:outer}

Anchored Learning proceeds in discrete outer iterations indexed by
$t = 0,1,2,\ldots, T$, where $T$ denotes the total number of outer iterations and the initial model ($t=0$) corresponding to the base distribution
$p_{\mathrm{base}}$. At each iteration, an anchor distribution is constructed by
interpolating between the current model
$p_{\theta^{(t)}}(\cdot\mid x)$ and the fixed SFT reference model
$p_{\mathrm{sft}}(\cdot\mid x)$:
\begin{equation}
q^{(t)}(\cdot\mid x)
=
\mathcal{I}_\alpha\!\left(
p_{\theta^{(t)}}(\cdot\mid x),\,
p_{\mathrm{sft}}(\cdot\mid x)
\right),
\label{eq:anchor_general}
\end{equation}
where $\mathcal{I}_\alpha$ denotes an interpolation operator with coefficient
$\alpha\in(0,1)$. Interpolating with the current model is essential: if
interpolation were performed only against a fixed reference, the anchor would be
static and the procedure would degenerate into conventional distillation with a
fixed teacher. The specific instantiations of $\mathcal{I}_\alpha$ are described
in Sec.~\ref{sec:interp}. Although we focus on interpolation between two models,
the framework naturally extends to multiple references.

\subsection{Inner Loop: Anchor Fitting via Distillation}
\label{sec:inner}

Given the anchor distribution $q^{(t)}$, the inner loop approximately projects
the current model onto the anchor by minimizing the distillation objective:
\begin{equation}
\theta^{(t+1)}
\;\approx\;
\arg\min_{\theta}
\;\mathcal{L}_{\mathrm{distill}}(\theta; q^{(t)}), ~ 
\text{initialized at } \theta^{(t)}.
\label{eq:inner_argmin}
\end{equation}
In practice, this minimization is carried out by running a gradient-based
optimizer for $K$ \emph{epochs} over the dataset $\mathcal{D}$ (e.g., AdamW or
SGD).
We implement the inner loop using a simple \emph{off-policy} knowledge distillation scheme on the fixed dataset $\mathcal{D}$.
The framework naturally extends to on-policy or hybrid distillation variants, which we leave to future work.

\subsection{Instantiations of the Interpolation Operator}
\label{sec:interp}

We instantiate the interpolation operator $\mathcal{I}_\alpha$ in two
representation spaces: logit space and probability space.

\paragraph{Logit-space interpolation.}
Let $z_{\theta^{(t)}}(x)$ and $z_{\mathrm{sft}}(x)$ denote the pre-softmax logits of
the current and SFT models, respectively. The interpolated logits are
\begin{equation}
z^{(t)}(x)
=
(1-\alpha)\, z_{\theta^{(t)}}(x)
+
\alpha\, z_{\mathrm{sft}}(x),
\label{eq:logit_interp}
\end{equation}
and the anchor distribution is obtained via
\begin{equation}
q^{(t)}(\cdot\mid x) = \mathrm{softmax}\!\big(z^{(t)}(x)\big).
\end{equation}

\paragraph{Probability-space interpolation.}
Interpolation can also be performed directly on predictive distributions:
\begin{equation}
q^{(t)}(\cdot\mid x)
=
(1-\alpha)\, p_{\theta^{(t)}}(\cdot\mid x)
+
\alpha\, p_{\mathrm{sft}}(\cdot\mid x),
\label{eq:prob_interp}
\end{equation}
which corresponds to a convex mixture of model predictions. We implicitly assume
that all distributions share the same support so that the KL divergence is
well-defined.

% \paragraph{Parameter-space interpolation.}
% Alternatively, interpolation can be applied in parameter space. Let
% $\theta^{(t)}$ and $\theta_{\mathrm{sft}}$ denote the parameters of the current and
% SFT models, respectively. The interpolated parameters are given by
% \begin{equation}
% \tilde{\theta}^{(t)}
% =
% (1-\alpha)\, \theta^{(t)}
% +
% \alpha\, \theta_{\mathrm{sft}}.
% \label{eq:param_interp}
% \end{equation}
% The anchor distribution is then defined as
% $q^{(t)}(\cdot\mid x) = p_{\tilde{\theta}^{(t)}}(\cdot\mid x)$.
% While parameter-space interpolation is simple and computationally efficient, it
% may not preserve functional smoothness in highly non-linear networks; we
% empirically evaluate its behavior in Sec.~X.

% More generally, the interpolation operator can be defined under alternative
% geometries, such as spherical or manifold-aware interpolation. We leave these
% extensions to future work.

\section{Theoretical Analysis}
\label{sec:theory}

We analyze the properties of the anchor distribution $q^{(t)}(\cdot|x)$ defined in
Eq.~\eqref{eq:anchor_general}. All results hold pointwise for any input
$(x, y) \in \mathcal{D}$. We provide proof sketches to highlight the key mechanisms,
while full derivations are deferred to Appendix~\ref{app:proofs}.

%-------------------------------------------------
\subsection{Bounded Divergence in Probability Space}
\label{sec:prob_bound}

We first establish that when using probability-space interpolation
(Eq.~\eqref{eq:prob_interp}), the anchor distribution bounds the update magnitude
relative to the current model, acting as a dynamic trust-region constraint.

\begin{lemma}[Divergence Bound for Probability Interpolation]
\label{lem:prob_bound}
Let the anchor $q^{(t)}$ be defined via probability-space interpolation as
\[
q^{(t)} = (1-\alpha)p_{\theta^{(t)}} + \alpha p_{\mathrm{sft}},
\qquad \alpha \in [0,1].
\]
Then,
\begin{equation}
\mathrm{KL}\big(q^{(t)}(\cdot|x) \,\|\, p_{\theta^{(t)}}(\cdot|x)\big)
\;\le\;
\alpha\,\mathrm{KL}\big(p_{\mathrm{sft}}(\cdot|x) \,\|\, p_{\theta^{(t)}}(\cdot|x)\big).
\end{equation}
\end{lemma}

\begin{proof}
The result follows from the convexity of the mapping
$q \mapsto \mathrm{KL}(q \,\|\, p_{\theta^{(t)}})$ and Jensen's inequality,
together with the identity
$\mathrm{KL}(p_{\theta^{(t)}} \,\|\, p_{\theta^{(t)}}) = 0$.
\end{proof}

%-------------------------------------------------
\subsection{Geometric Interpretation of Logit Interpolation}
\label{sec:logit_geom}

When using logit-space interpolation (Eq.~\eqref{eq:logit_interp}), the resulting
anchor distribution takes the form of a renormalized geometric mean:
\begin{equation}
q^{(t)}_y(x)
=
\frac{p_{\theta^{(t)}}(y|x)^{1-\alpha} \, p_{\mathrm{sft}}(y|x)^{\alpha}}{Z(x)},
\label{eq:geom_interp}
\end{equation}
where $Z(x)$ is the partition function. This form admits a natural geometric
interpretation.

\begin{lemma}[Reverse-KL Barycenter / Geodesic Path]
\label{lem:barycenter}
Let $\alpha \in (0,1)$ and define
\begin{equation}
\mathcal{J}(p)
\;\triangleq\;
(1-\alpha)\,\mathrm{KL}\big(p \,\|\, p_{\theta^{(t)}}(\cdot|x)\big)
+
\alpha\,\mathrm{KL}\big(p \,\|\, p_{\mathrm{sft}}(\cdot|x)\big).
\label{eq:def_J}
\end{equation}
Then the anchor distribution $q^{(t)}$ in Eq.~\eqref{eq:geom_interp} satisfies
\begin{equation}
q^{(t)}(\cdot|x)
=\argmin_{p \in \Delta} \; \mathcal{J}(p).
\end{equation}
\end{lemma}

\begin{proof}
The result follows from solving the first-order optimality conditions of the
constrained problem via a Lagrangian formulation; a complete derivation is
provided in Appendix~\ref{app:proofs}.
\end{proof}

\paragraph{Remark on Information Geometry.}
Lemma~\ref{lem:barycenter} implies that $q^{(t)}$ lies on the exponential geodesic
($e$-geodesic) connecting $p_{\theta^{(t)}}$ and $p_{\mathrm{sft}}$, equivalently
corresponding to linear interpolation in log-probability (logit) space followed
by normalization. This suggests that logit-space Anchored Learning encourages the
model to move along a smooth and statistically natural path toward the SFT
reference.

\begin{lemma}[Bound for Logit Interpolation]
\label{lem:logit_bound}
Let $\alpha \in [0,1)$. For the anchor distribution $q^{(t)}$ defined in
Eq.~\eqref{eq:geom_interp}, the divergence from the current model satisfies
\begin{equation}
\mathrm{KL}\big(q^{(t)}(\cdot|x) \,\|\, p_{\theta^{(t)}}(\cdot|x)\big)
\le
\frac{\alpha}{1-\alpha}\,
\mathrm{KL}\big(p_{\theta^{(t)}}(\cdot|x) \,\|\, p_{\mathrm{sft}}(\cdot|x)\big).
\end{equation}
\end{lemma}

\begin{proof}
By the optimality of $q^{(t)}$ in Lemma~\ref{lem:barycenter}, the weighted sum of
divergences evaluated at $q^{(t)}$ is no larger than that evaluated at any other
distribution, in particular $p_{\theta^{(t)}}$. Expanding the objective, dropping
the non-negative term $\mathrm{KL}(q^{(t)} \,\|\, p_{\mathrm{sft}})$, and
rearranging terms yields the stated bound.
\end{proof}

%-------------------------------------------------
\subsection{Trust-Region Interpretation and Dynamic Stability}
\label{sec:trust_region}

Lemma~\ref{lem:prob_bound} and Lemma~\ref{lem:logit_bound} show that under both
probability-space and logit-space interpolation, the anchor distribution
$q^{(t)}$ stays within a controlled divergence from the current model
$p_{\theta^{(t)}}$. Consequently, minimizing the distillation objective
$\mathrm{KL}(q^{(t)} \,\|\, p_\theta)$ induces a \emph{local} and bounded behavioral
update at each outer iteration, which can be interpreted as an implicit
trust-region constraint in distribution space.

Crucially, the anchor is constructed relative to the \emph{current} model rather
than a fixed reference. This dynamic coupling makes the target distribution
co-evolve with the optimization trajectory, preventing the training objective from
imposing a persistent attraction toward a static and potentially misaligned base
distribution once partial adaptation has occurred. As a result, Anchored Learning
reduces the optimization tension that commonly arises when simultaneously fitting
the new task and preserving the base behavior.

To contrast, consider static KL regularization toward the base model, which
implicitly optimizes a fixed compromise objective of the form
\begin{equation}
\min_{p \in \Delta}
\;
(1-\beta)\,\mathrm{KL}\big(p \,\|\, p_{\mathrm{sft}}\big)
+
\beta\,\mathrm{KL}\big(p \,\|\, p_{\mathrm{base}}\big),
\qquad \beta \in (0,1).
\label{eq:static_compromise}
\end{equation}
The optimizer of Eq.~\eqref{eq:static_compromise} is a single barycenter between
$p_{\mathrm{sft}}$ and $p_{\mathrm{base}}$, independent of the current model state.
Such a static target can remain misaligned with the desired adaptation trajectory
after the model has already moved toward the task distribution, thereby inducing
persistent gradient conflicts between fitting and preservation objectives.

In contrast, Anchored Learning realizes a sequence of moving barycenters that
continuously interpolate between the current model and the SFT reference.

\begin{proposition}[Dynamic Convergence of Anchored Learning]
\label{prop:dynamic_static}
Assume that the inner-loop optimization exactly matches the anchor distribution at
each outer iteration, i.e., $p_{\theta^{(t+1)}} = q^{(t)}$, and that the anchor is
constructed via probability-space interpolation:
\[
q^{(t)} = (1-\alpha)p_{\theta^{(t)}} + \alpha p_{\mathrm{sft}},
\qquad \alpha \in (0,1).
\]
Then the induced recursion satisfies
\[
p_{\theta^{(t+1)}}
=
(1-\alpha)p_{\theta^{(t)}} + \alpha p_{\mathrm{sft}},
\]
and converges geometrically to $p_{\mathrm{sft}}$ as $t \to \infty$.
Moreover,
\[
\mathrm{KL}\big(p_{\theta^{(t)}} \,\|\, p_{\mathrm{sft}}\big)
\le
(1-\alpha)^t \,
\mathrm{KL}\big(p_{\mathrm{base}} \,\|\, p_{\mathrm{sft}}\big).
\]
\end{proposition}

Under the exact-projection assumption, the resulting recursion converges
geometrically to $p_{\mathrm{sft}}$ while maintaining bounded stepwise deviation;
detailed derivations are provided in Appendix~\ref{app:dynamic_static}. This
perspective explains why Anchored Learning yields stable progressive adaptation
whereas naive static regularization often fails to achieve a favorable trade-off
in practice.

\section{Experiments}
\label{sec:exp}

\subsection{Experimental Setup}
\paragraph{Datasets.}
We evaluate Anchored Learning under task-specific fine-tuning settings across three representative target tasks:
\textbf{iGSM}~\cite{ye2025physics}, a synthetic grade-school mathematics benchmark with controllable difficulty;
\textbf{MedCalc}~\cite{khandekar2024medcalc}, a benchmark for medical calculation and numerical reasoning;
and \textbf{IFEval}~\cite{zhou2023instructionfollowingevaluationlargelanguage}, which evaluates instruction-following under verifiable constraints.
For each task, models are fine-tuned or distilled on the official training split, hyperparameters are selected based on the corresponding validation set when available, and final performance is reported on the test set.

To quantify catastrophic forgetting, we additionally evaluate all finetuned models on a suite of general-purpose benchmarks that are not observed during target-task training.
These include \textbf{MMLU-Pro}~\cite{wang2024mmlu} for knowledge and reasoning,
\textbf{Countdown}~\cite{tinyzero} for mathematical reasoning,
and \textbf{MBPP}, \textbf{MBPP+}, \textbf{HumanEval}, and \textbf{HumanEval+}~\cite{liu2023is} for code generation.
This evaluation protocol enables simultaneous assessment of target-task gains and degradation of prior capabilities.

\paragraph{Models.}
We evaluate our method on two diverse model families to assess robustness across architectures and scales:
\textbf{Qwen2.5-Instruct}~\cite{qwen2025qwen25technicalreport} at scales of 1.5B, 3B, 7B, and 14B parameters,
and \textbf{Llama-3.2-3B-Instruct}~\cite{grattafiori2024llama}.
Unless otherwise specified, the Qwen2.5-3B-Instruct and Llama-3.2-3B-Instruct models serve as the primary testbeds for baseline comparisons due to their comparable parameter counts.

\begin{figure*}[ht]
    \centering
    \includegraphics[width=\linewidth]{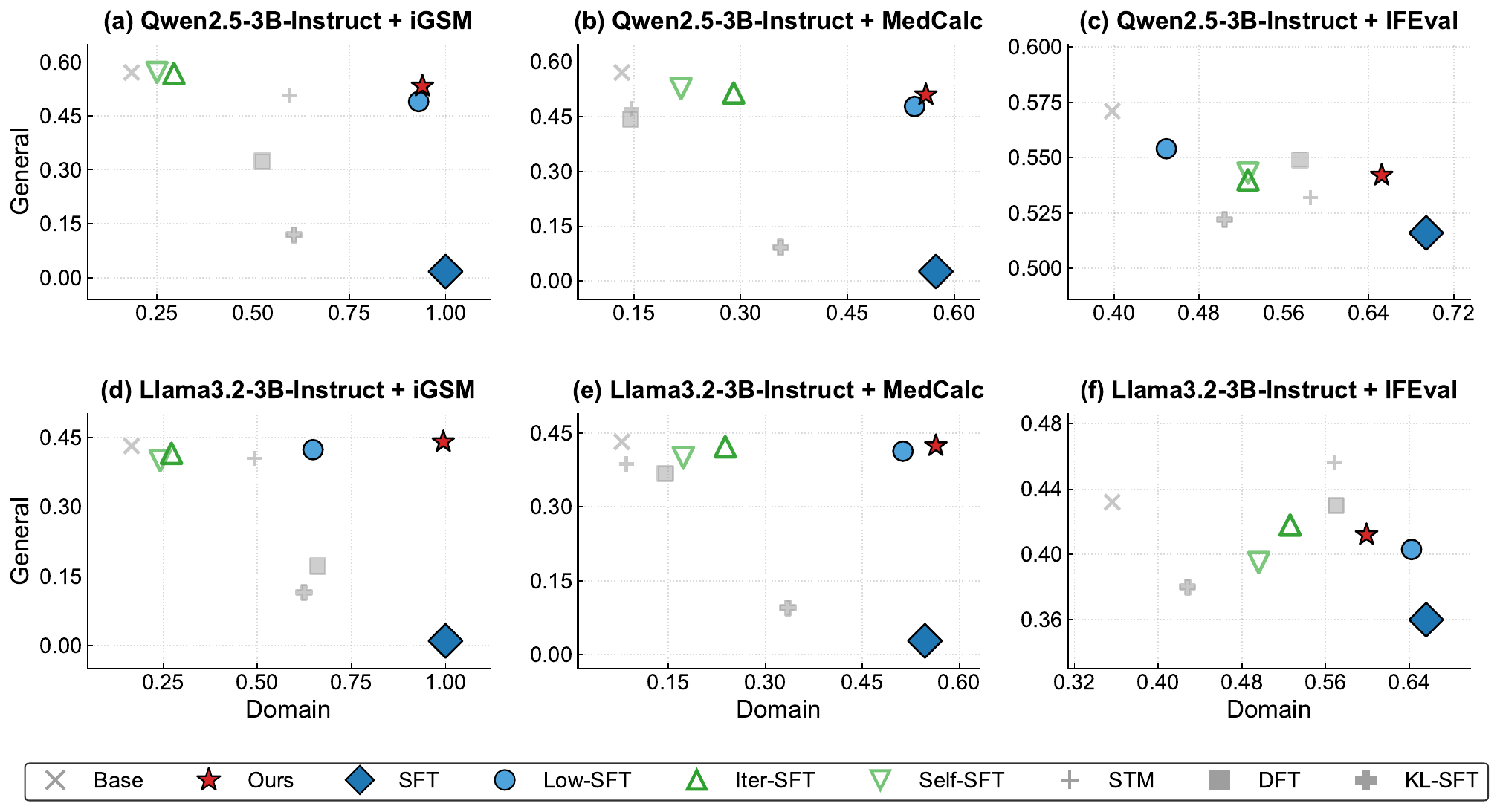}
   \caption{Performance trade-offs between domain performance (x-axis) and general performance (y-axis) across three target tasks (iGSM, MedCalc, and IFEval) for Qwen2.5-3B-Instruct (top row) and Llama3.2-3B-Instruct (bottom row). Each point corresponds to a different fine-tuning method. Points closer to the upper-right corner indicate better performance--stability trade-offs. Anchored Learning (Ours) consistently lies near the Pareto frontier, achieving strong domain gains while preserving general capabilities.}

    \label{fig:exp_main_results1}
\end{figure*}

\paragraph{Training Protocol.}
For all methods, fine-tuning or distillation is performed on the same target-task training data using identical optimization budgets to ensure fair comparison.
Unless otherwise specified, all models are initialized from the corresponding instruction-tuned checkpoints.
Hyperparameters for baseline methods follow the settings reported in their original papers or are tuned within comparable search ranges.
Unless otherwise specified, all experiments use Qwen2.5-3B-Instruct as the default model, with logit-space interpolation, an interpolation coefficient $\alpha=0.5$, $K=5$ inner-loop iterations per outer iteration, and a total of $T=5$ outer iterations.
% Additional implementation details are provided in Appendix~\ref{sec:implementation}.

\paragraph{Evaluation Metrics.}
We report \textbf{Accuracy} for all non-coding benchmarks and \textbf{Pass@1} for code generation tasks.
To summarize general capability retention, we additionally report a consolidated \textit{General Average} score computed across all evaluation benchmarks.

\paragraph{Baselines.}
We compare Anchored Learning against standard \textbf{SFT} (full fine-tuning, learning rate $1\times10^{-5}$) and a diverse set of forgetting-mitigation baselines.
These include \textbf{Low-SFT}~\cite{pareja2025unveiling,lin2025sft}, which applies a substantially reduced learning rate ($1\times10^{-6}$) to preserve general capabilities;
\textbf{KL-SFT}, which adds a KL penalty with coefficient $\lambda=0.2$ to constrain the fine-tuned model toward a fixed reference distribution; 
\textbf{Self-SFT}~\cite{zelikman2022star,chen2025retaining}, which bootstraps training data from model-generated responses filtered by a ground-truth reward function;
\textbf{Iter-SFT}~\cite{chen2025retaining}, which leverages approximately on-policy data generation to reduce distributional shift;
\textbf{STM}~\cite{wu2025mitigating}, which masks high-perplexity tokens to reduce the influence of outliers during optimization;
and \textbf{DFT}~\cite{wu2025generalization}, which dynamically rescales the training objective based on token probabilities to stabilize gradient updates.
All baselines are trained under comparable optimization budgets and follow the recommended hyperparameter settings from their original works when applicable.

\subsection{Main Results}

Fig.~\ref{fig:exp_main_results1} summarizes the trade-offs between domain performance and general performance across six model--task settings. Three consistent observations emerge.

\textbf{First}, standard SFT achieves the highest domain performance but suffers from severe degradation in general capabilities, often collapsing toward the lower-right corner of the plots. This confirms that unconstrained fine-tuning induces substantial distributional drift and catastrophic forgetting.

\textbf{Second}, existing mitigation baselines alleviate forgetting to varying degrees but typically sacrifice substantial domain performance, leading to points clustered toward the upper-left or central regions. This highlights the inherent difficulty of achieving both strong adaptation and stable retention using static regularization or data-centric heuristics.

\textbf{Third}, Anchored Learning consistently lies near the Pareto frontier across all settings, achieving strong domain gains while maintaining general performance close to the base model. This demonstrates that explicitly controlling the optimization trajectory via a moving anchor enables a superior performance--stability trade-off compared to both aggressive fine-tuning and conservative regularization baselines.

\subsection{Scaling Analysis}

\begin{table}[t]
\centering
\renewcommand{\arraystretch}{1.05}
\caption{Scaling comparison across Qwen2.5 model sizes on iGSM. 
We report general performance and domain performance on iGSM.
Anchored Learning consistently preserves higher general capabilities than Low-SFT while achieving comparable domain performance across all scales.}
\label{tab:scaling}

\begin{tabular}{llccc}
\toprule
\textbf{Data} & \textbf{Method} & \textbf{1.5B} & \textbf{3B} & \textbf{7B} \\
\midrule
\multirow{3}{*}{\emph{Domain}}
& Base      & 0.144 & 0.184 & 0.228 \\
& Low-SFT   & 0.954 & 0.930 & 0.970 \\
& {Ours} & 0.950 & {0.936} & 0.968 \\
\midrule
\multirow{3}{*}{\emph{General}}
& Base      & 0.452 & 0.571 & 0.674 \\
& Low-SFT   & 0.374 & 0.496 & 0.592 \\
& {Ours} & {0.395} & {0.533} & {0.616} \\
\bottomrule
\end{tabular}
\end{table}

\paragraph{Scaling behavior across model sizes.}
Tab.~\ref{tab:scaling} reports performance on iGSM across three Qwen2.5 model scales (1.5B, 3B, and 7B), comparing general capability retention and domain performance.
Across all scales, Anchored Learning consistently preserves substantially higher general performance than Low-SFT, while achieving comparable domain accuracy.
Specifically, on the smallest 1.5B model, Anchored Learning improves general performance from 0.374 (Low-SFT) to 0.395, while maintaining nearly identical iGSM accuracy (0.950 vs.\ 0.954).
This trend persists at 3B and 7B scales, where Anchored Learning recovers a significant portion of the general capability gap to the base model, improving over Low-SFT by +3.7 and +2.4 points respectively, without sacrificing domain performance.
Notably, the domain accuracy of Anchored Learning closely matches that of Low-SFT across all sizes, indicating that stability gains are not obtained at the cost of reduced adaptation effectiveness.
These results demonstrate that the benefits of Anchored Learning are robust to model scaling and do not rely on a specific parameter regime.
The consistent improvement in general retention across model sizes suggests that explicit control of distributional updates provides a scalable mechanism for stabilizing offline fine-tuning.

\subsection{Other Analysis}
We conduct ablation studies on Qwen2.5-3B-Instruct with \textbf{MedCalc} as the target task to analyze the contribution of different components and hyperparameters.

\begin{figure}[t]
  \vskip 0.2in
  \begin{center}
    \centerline{\includegraphics[width=\columnwidth]{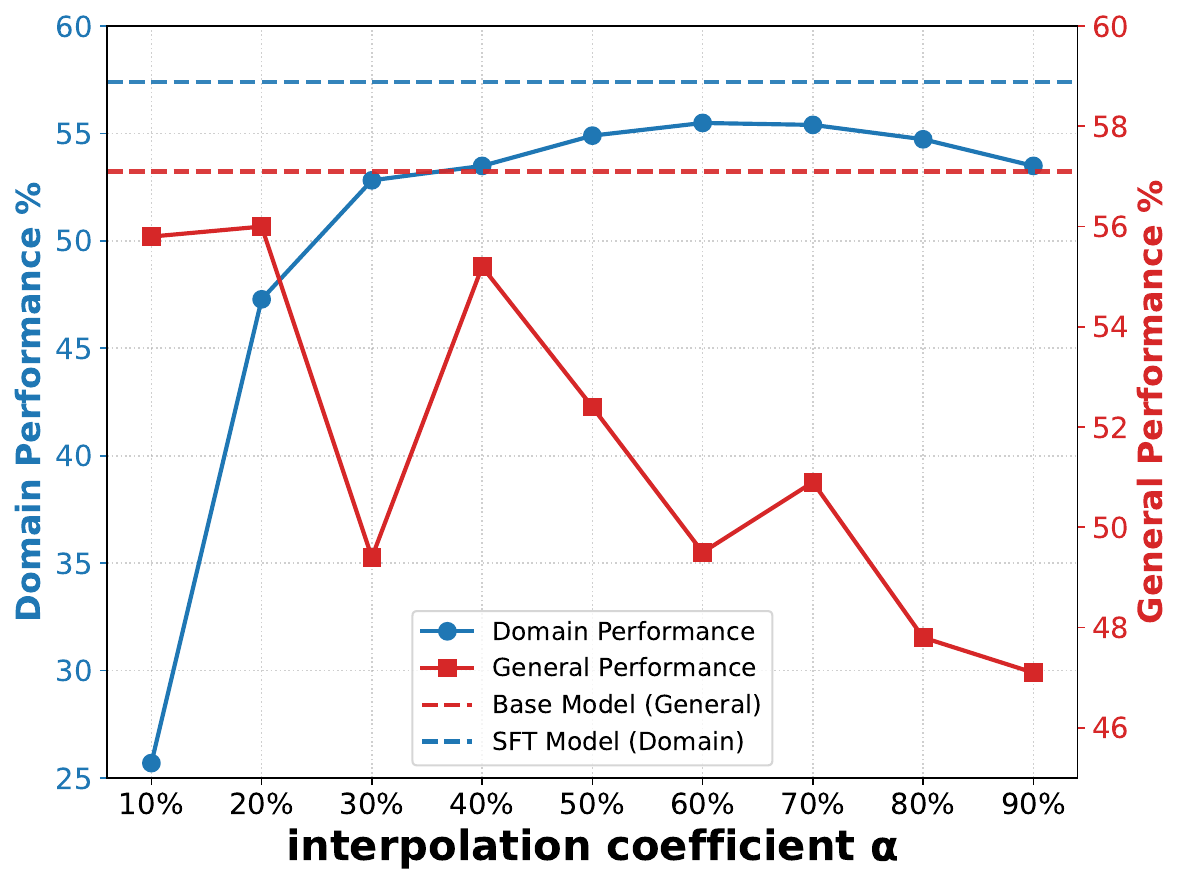}}
\caption{Domain (left) and general (right) performance versus the interpolation coefficient $\alpha$ for Qwen2.5-3B-Instruct on MedCalc. Dashed lines denote the SFT and base reference levels, illustrating a performance--stability trade-off.}
    \label{fig:exp_results_alpha}
  \end{center}
\end{figure}
\paragraph{Sensitivity to the interpolation coefficient.}
Fig.~\ref{fig:exp_results_alpha} illustrates the effect of the interpolation coefficient on domain performance and general performance.
Domain performance increases rapidly at small coefficients and quickly saturates, indicating diminishing returns from more aggressive interpolation.
In contrast, general performance gradually degrades and exhibits non-monotonic variations as the coefficient increases, reflecting the heightened sensitivity of stability to update strength.
An intermediate range of coefficients achieves a favorable performance--stability trade-off, motivating the use of controlled interpolation to balance adaptation and retention.

\begin{figure}[h]
  \vskip 0.2in
  \begin{center}
    \centerline{\includegraphics[width=0.9\columnwidth]{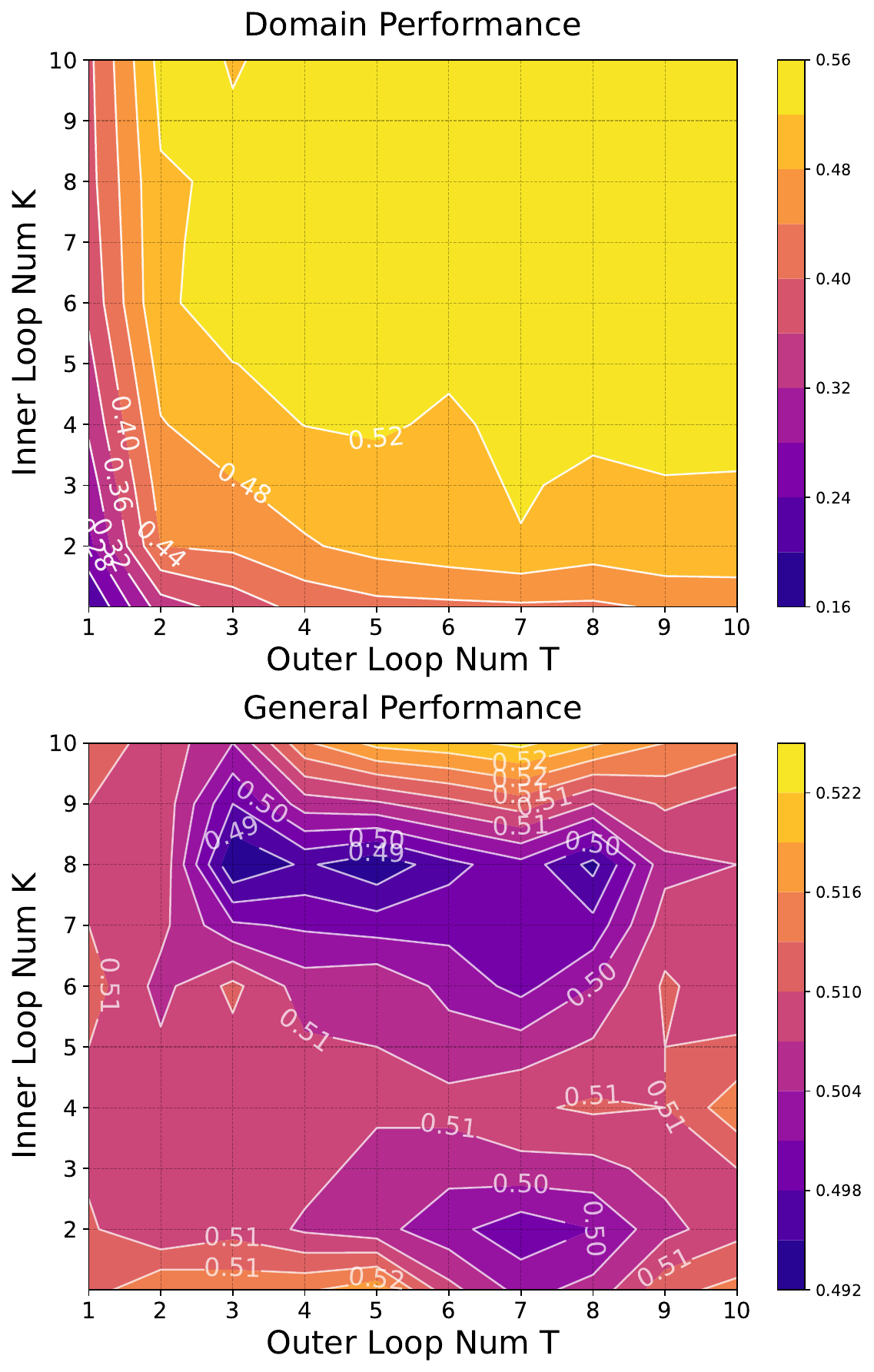}}
\caption{Sensitivity of target and general performance to the outer and inner loop iterations on Qwen2.5-3B-Instruct with MedCalc as the target task. General performance exhibits non-monotonic degradation despite saturated target accuracy.}
    \label{fig:exp_results_in_out}
  \end{center}
\end{figure}
\paragraph{Sensitivity to update schedule.}
Fig.~\ref{fig:exp_results_in_out} examines how the update schedule, parameterized by the numbers of outer and inner loop iterations, affects (i) target-domain performance and (ii) general capability retention.
First, the two objectives respond fundamentally differently to the same schedule changes.
Target-domain performance increases almost monotonically with the number of outer iterations and quickly reaches a saturation regime (top panel), whereas general performance exhibits clear non-monotonic behavior with structured variations across the grid (bottom panel).
This indicates that optimizing for target accuracy alone can be misleading, as stability is governed by path-dependent effects rather than a simple notion of ``more training is better.''
Second, high target accuracy does not necessarily imply stable adaptation.
In several regions, the model already attains near-saturated domain performance while general performance degrades noticeably, forming distinct degradation basins.
This suggests that aggressive or poorly balanced schedules can induce stability loss even when the target objective appears well optimized, consistent with the view that catastrophic forgetting is driven by excessive distributional drift along the optimization trajectory.
Third, these degradation regions are localized and highly sensitive to the schedule, implying that static hyperparameter choices are unlikely to reliably avoid instability.
Small changes in the inner or outer iteration counts can move the training dynamics across these basins, leading to markedly different general retention.
Together, these observations motivate the need for explicit trajectory control mechanisms, rather than relying solely on fixed schedules or globally tuned regularizers, to achieve robust gain--stability trade-offs in offline fine-tuning.

\paragraph{Sensitivity to interpolation operator.}
Fig.~\ref{fig:exp_results_logit_prob} highlights a distinct plasticity-stability trade-off.
\begin{figure}
    \centering
    \includegraphics[width=\linewidth]{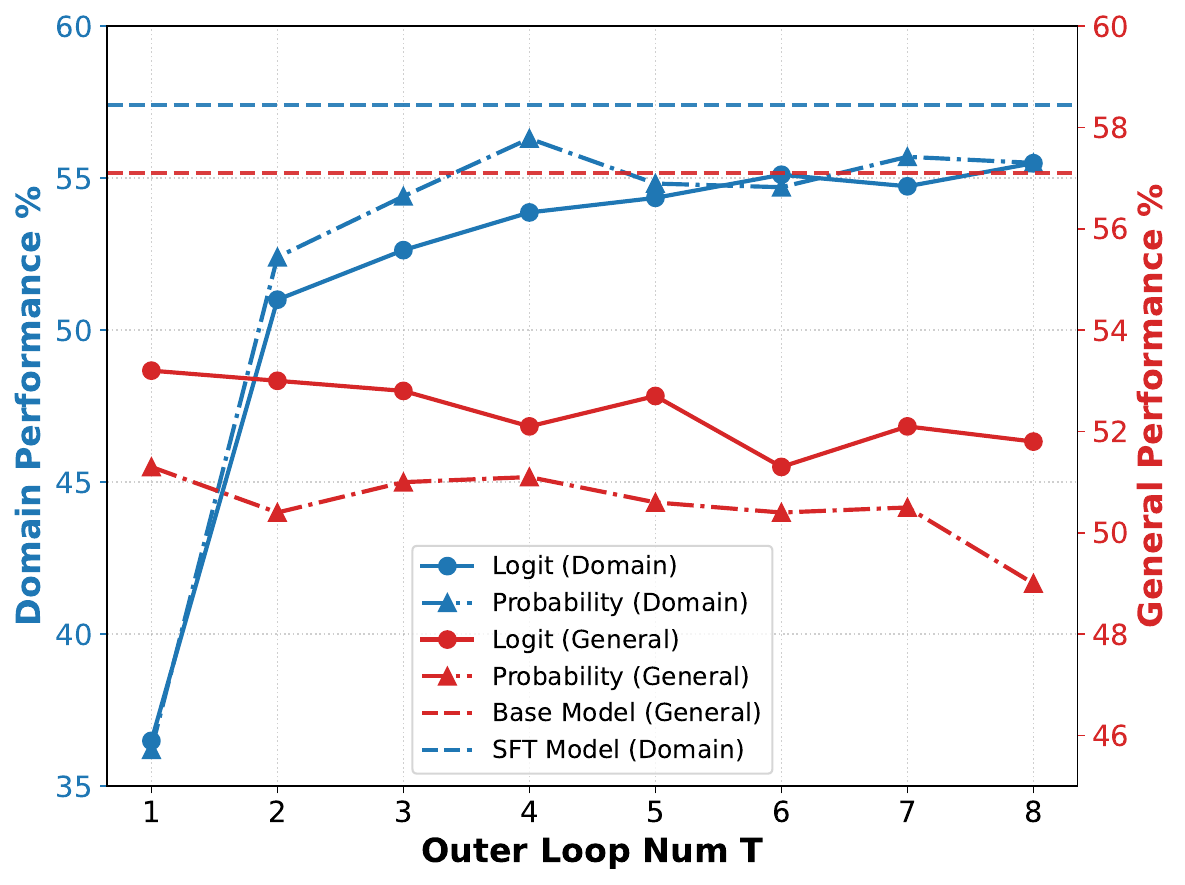}
    \caption{Difference of the interpolation space (logit vs. probability) on domain and general performance across outer loop iterations for Qwen2.5-3B-Instruct on MedCalc.}
    \label{fig:exp_results_logit_prob}
\end{figure}
The performance divergence stems from the mathematical properties of the averaging methods. Probability-space interpolation operates as an arithmetic mean, inducing mode-covering behavior that favors plasticity. This yields superior domain adaptation (Blue-dashed) but higher forgetting. In contrast, Logit-space interpolation acts as a geometric mean (after renormalization), inducing mode-seeking behavior. This conservative aggregation favors stability, better preserving general capabilities (Red-solid) by suppressing low-probability tails.

\section{Related Work}
\label{sec:related}
\paragraph{Post-training Stability and Forgetting.}
Catastrophic forgetting is a long-standing challenge in neural network adaptation~\cite{kirkpatrick2017overcoming}, and has recently been revisited in the context of LLM post-training. 
Recent comparisons between RL and SFT suggest that on-policy RL tends to forget less while achieving comparable target performance, and that forgetting correlates with distributional shift measured by KL divergence to the base policy~\cite{shenfeld2025rl}. 
Complementarily, the robustness of RL has been attributed to the mode-seeking effect induced by (approximately) on-policy data, rather than specific algorithmic components such as KL regularization or advantage estimation~\cite{chen2025retaining}. 
Our work builds on these insights but targets a different setting: we focus on purely offline SFT where general-domain data is unavailable, and aim to explicitly control distributional drift through the design of the optimization trajectory.

\paragraph{Stabilizing Supervised Fine-Tuning.}
A broad line of work seeks to stabilize SFT and mitigate capability degradation through optimization- or data-centric interventions. 
On the optimization side, reducing the learning rate has been shown to improve retention and preserve general capabilities~\cite{pareja2025unveiling,lin2025sft}, while token- or objective-level modifications such as low-perplexity token learning and dynamic objective rescaling can improve robustness~\cite{wu2025mitigating,wu2025generalization}. 
On the data side, self-training and self-distillation approaches aim to narrow distribution gaps by distilling from model-generated or teacher outputs~\cite{zelikman2022star,yang2024self}, and approximately on-policy data generation has also been explored to reduce forgetting~\cite{chen2025retaining}. 
In contrast to these approaches, Anchored Learning introduces a dynamically evolving intermediate target (a moving anchor) to explicitly control distributional updates during offline fine-tuning, with theoretical per-iteration KL stability guarantees.

\section{Conclusion}
\label{sec:conclusion}

We presented \emph{Anchored Learning}, a simple framework that explicitly controls distributional updates during offline fine-tuning through a dynamically evolving moving anchor. 
By shaping intermediate target distributions, Anchored Learning transforms global optimization into a sequence of local and conservative updates, yielding implicit trust-region behavior with minimal overhead.
We established theoretical guarantees on distributional stability via explicit per-iteration KL-divergence bounds and demonstrated empirically that the method substantially reduces catastrophic forgetting while maintaining strong target-task performance across multiple benchmarks. 
These results highlight the importance of explicit trajectory control for stable post-training and suggest a practical alternative to implicit stability mechanisms in on-policy reinforcement learning.

\section*{Limitation}
\label{sec:limit}
The primary limitation of Anchored Learning is the additional computational overhead compared to standard SFT. 
Each iteration requires constructing an anchored target via interpolation and an extra forward pass for distillation, leading to moderately higher training latency and memory usage, typically about $1.5\times$--$2\times$ in our experiments. 
Although this overhead is substantially lower than that of on-policy reinforcement learning, it may limit scalability in highly resource-constrained settings.

% Acknowledgements should only appear in the accepted version.
% \section*{Acknowledgements}

% \section*{Impact Statement}
% \input{100impact_statment}

\bibliography{main}
\bibliographystyle{icml2026}

%%%%%%%%%%%%%%%%%%%%%%%%%%%%%%%%%%%%%%%%%%%%%%%%%%%%%%%%%%%%%%%%%%%%%%%%%%%%%%%
%%%%%%%%%%%%%%%%%%%%%%%%%%%%%%%%%%%%%%%%%%%%%%%%%%%%%%%%%%%%%%%%%%%%%%%%%%%%%%%
% APPENDIX
%%%%%%%%%%%%%%%%%%%%%%%%%%%%%%%%%%%%%%%%%%%%%%%%%%%%%%%%%%%%%%%%%%%%%%%%%%%%%%%
%%%%%%%%%%%%%%%%%%%%%%%%%%%%%%%%%%%%%%%%%%%%%%%%%%%%%%%%%%%%%%%%%%%%%%%%%%%%%%%
\newpage
\appendix
\onecolumn

\section{Detailed Theoretical Proofs}
\label{app:proofs}

In this appendix, we provide rigorous derivations for the properties of the
Anchored Learning target distributions discussed in
Section~\ref{sec:theory}. Throughout this appendix, all lemma and proposition
numbers refer to those in Section~\ref{sec:theory}. All arguments are carried out
pointwise for a fixed input $x$; for notational simplicity, we omit the
conditioning on $x$ and write $p_{\theta^{(t)}}$ in place of
$p_{\theta^{(t)}}(\cdot \mid x)$ whenever no ambiguity arises.

\paragraph{Standing Assumption.}
All model distributions are induced by a softmax parameterization and therefore
have full support on the discrete output space, i.e., $p_i > 0$ for all tokens $i$.
Consequently, all Kullback--Leibler (KL) divergences appearing below are finite and
well-defined, and all logarithms are taken over strictly positive arguments.
Moreover, this assumption guarantees that all optimal solutions considered below
lie in the interior of the probability simplex.

%-------------------------------------------------------------------------
\subsection{Probability-Space Interpolation (Proof of Lemma~\ref{lem:prob_bound})}

We first establish the divergence bound under linear interpolation in probability
space.

\paragraph{Restatement of Lemma~\ref{lem:prob_bound}.}
\textit{Let the anchor be defined as}
\[
q^{(t)} = (1-\alpha)p_{\theta^{(t)}} + \alpha p_{\mathrm{sft}}, 
\quad \alpha \in [0,1].
\]
\textit{Then}
\[
\mathrm{KL}\big(q^{(t)} \,\|\, p_{\theta^{(t)}}\big)
\le 
\alpha \, \mathrm{KL}\big(p_{\mathrm{sft}} \,\|\, p_{\theta^{(t)}}\big).
\]

\begin{proof}
For a fixed reference distribution $p_{\theta^{(t)}}$, define the functional
\[
F(u) \triangleq \mathrm{KL}(u \| p_{\theta^{(t)}})
= \sum_i u_i \log u_i - \sum_i u_i \log p_{\theta^{(t)},i}.
\]
The first term is the negative entropy, which is convex in $u$ since the scalar
function $f(t)=t\log t$ satisfies $f''(t)=1/t>0$ for all $t>0$. The second term is
linear in $u$. Hence, $F$ is convex on the probability simplex.

By Jensen's inequality, for any distributions $u_1,u_2$ and $\lambda \in [0,1]$,
\[
F\big((1-\lambda)u_1 + \lambda u_2\big)
\le (1-\lambda)F(u_1) + \lambda F(u_2).
\]
Instantiating this inequality with $u_1 = p_{\theta^{(t)}}$,
$u_2 = p_{\mathrm{sft}}$, and $\lambda = \alpha$, and noting that
$\mathrm{KL}(p_{\theta^{(t)}} \| p_{\theta^{(t)}})=0$, yields
\[
\mathrm{KL}\big((1-\alpha)p_{\theta^{(t)}} + \alpha p_{\mathrm{sft}}
\,\|\, p_{\theta^{(t)}}\big)
\le 
\alpha \, \mathrm{KL}\big(p_{\mathrm{sft}} \,\|\, p_{\theta^{(t)}}\big).
\]
Substituting the definition of $q^{(t)}$ completes the proof. \qedhere
\end{proof}

%-------------------------------------------------------------------------
\subsection{Logit-Space Interpolation (Proof of Lemma~\ref{lem:barycenter})}

We next derive the closed-form solution of the reverse-KL barycenter and show that
it coincides with geometric-mean (logit-space) interpolation.

\paragraph{Restatement of Lemma~\ref{lem:barycenter}.}
\textit{Let $\Delta$ denote the probability simplex. The optimization problem}
\[
\min_{u \in \Delta} 
\; 
\mathcal{J}(u)
\quad \text{where} \quad 
\mathcal{J}(u) 
= (1-\alpha)\mathrm{KL}(u\|p_{\theta^{(t)}}) 
+ \alpha \mathrm{KL}(u\|p_{\mathrm{sft}}),
\]
\textit{admits the unique minimizer}
\[
q^{(t)}_i = \frac{1}{Z} \, 
p_{\theta^{(t)},i}^{\,1-\alpha} \, p_{\mathrm{sft},i}^{\,\alpha},
\qquad 
Z = \sum_j p_{\theta^{(t)},j}^{\,1-\alpha} p_{\mathrm{sft},j}^{\,\alpha}.
\]

\begin{proof}
We enforce the simplex constraint $\sum_i u_i = 1$ via a Lagrange multiplier
$\lambda$. The Lagrangian is
\begin{align*}
\mathcal{L}(u,\lambda)
&= (1-\alpha)\sum_i u_i \log \frac{u_i}{p_{\theta^{(t)},i}}
+ \alpha \sum_i u_i \log \frac{u_i}{p_{\mathrm{sft},i}}
+ \lambda \Big(\sum_i u_i - 1\Big) \\
&= \sum_i u_i 
\Big[
\log u_i 
- (1-\alpha)\log p_{\theta^{(t)},i}
- \alpha \log p_{\mathrm{sft},i}
+ \lambda
\Big] - \lambda .
\end{align*}
For each coordinate $k$, using $\frac{d}{dx}(x\log x)=1+\log x$, we obtain
\[
\frac{\partial \mathcal{L}}{\partial u_k}
= 1 + \log u_k 
- \log\!\big(p_{\theta^{(t)},k}^{1-\alpha}
p_{\mathrm{sft},k}^{\alpha}\big)
+ \lambda .
\]
Setting this derivative to zero yields
\[
\log u_k
= \log\!\big(p_{\theta^{(t)},k}^{1-\alpha}
p_{\mathrm{sft},k}^{\alpha}\big) - (1+\lambda),
\qquad
u_k
= C \, p_{\theta^{(t)},k}^{1-\alpha} p_{\mathrm{sft},k}^{\alpha},
\]
where $C = e^{-(1+\lambda)}$ is a normalization constant. Enforcing $\sum_k u_k = 1$
gives
\[
C^{-1} = \sum_j p_{\theta^{(t)},j}^{1-\alpha} p_{\mathrm{sft},j}^{\alpha}
\eqqcolon Z,
\]
and therefore
\[
u_k^* = \frac{p_{\theta^{(t)},k}^{1-\alpha}
p_{\mathrm{sft},k}^{\alpha}}{Z}.
\]

Since both reference distributions have full support and $\alpha \in (0,1)$, the
objective $\mathcal{J}$ is strictly convex on the interior of the simplex. Hence,
the stationary point characterized above is the unique global minimizer. Taking
logarithms shows that
\[
\log q^{(t)}_i =
(1-\alpha)\log p_{\theta^{(t)},i}
+ \alpha \log p_{\mathrm{sft},i}
- \log Z,
\]
which corresponds to linear interpolation in log-probability (logit) space.
\qedhere
\end{proof}

%-------------------------------------------------------------------------
\subsection{Logit-Space Divergence Bound (Proof of Lemma~\ref{lem:logit_bound})}

We finally establish an upper bound on the divergence between the logit-interpolated
anchor and the current model distribution.

\paragraph{Restatement of Lemma~\ref{lem:logit_bound}.}
\textit{Let $\alpha \in [0,1)$. For the anchor $q^{(t)}$ defined by logit
interpolation,}
\[
\mathrm{KL}(q^{(t)} \| p_{\theta^{(t)}}) 
\le 
\frac{\alpha}{1-\alpha} 
\, \mathrm{KL}(p_{\theta^{(t)}} \| p_{\mathrm{sft}}).
\]

\begin{proof}
Let $\mathcal{J}(u)$ denote the objective defined in
Lemma~\ref{lem:barycenter}. Since $q^{(t)}$ is the global minimizer of $\mathcal{J}$
over $\Delta$, for any $u' \in \Delta$,
\[
\mathcal{J}(q^{(t)}) \le \mathcal{J}(u').
\]
Choosing $u' = p_{\theta^{(t)}}$ gives
\[
\mathcal{J}(p_{\theta^{(t)}})
= (1-\alpha)\mathrm{KL}(p_{\theta^{(t)}}\|p_{\theta^{(t)}})
+ \alpha \mathrm{KL}(p_{\theta^{(t)}}\|p_{\mathrm{sft}})
= \alpha \mathrm{KL}(p_{\theta^{(t)}}\|p_{\mathrm{sft}}).
\]
On the other hand,
\[
\mathcal{J}(q^{(t)})
= (1-\alpha)\mathrm{KL}(q^{(t)}\|p_{\theta^{(t)}})
+ \alpha \mathrm{KL}(q^{(t)}\|p_{\mathrm{sft}}).
\]
Combining the two displays and using the non-negativity of KL divergence yields
\[
(1-\alpha)\mathrm{KL}(q^{(t)}\|p_{\theta^{(t)}})
\le 
\alpha \mathrm{KL}(p_{\theta^{(t)}}\|p_{\mathrm{sft}}).
\]
Dividing both sides by $1-\alpha$ (which is strictly positive) completes the proof.
\qedhere
\end{proof}

%-------------------------------------------------------------------------
\subsection{Dynamic Versus Static Barycenters (Proof of Proposition~\ref{prop:dynamic_static})}
\label{app:dynamic_static}

We provide detailed derivations for the comparison between static KL barycenters
and the dynamic anchor recursion used in Anchored Learning.

\paragraph{Dynamic Anchor Recursion.}
Assume that the inner-loop optimization exactly matches the anchor distribution at
each outer iteration, i.e.,
\begin{equation}
p_{\theta^{(t+1)}} = q^{(t)}.
\label{eq:exact_projection}
\end{equation}
For probability-space interpolation, the anchor is defined as
\begin{equation}
q^{(t)} = (1-\alpha)p_{\theta^{(t)}} + \alpha p_{\mathrm{sft}},
\qquad \alpha \in (0,1).
\label{eq:dynamic_anchor_def}
\end{equation}
Substituting Eq.~\eqref{eq:dynamic_anchor_def} into
Eq.~\eqref{eq:exact_projection} yields the linear recursion
\begin{equation}
p_{\theta^{(t+1)}} =
(1-\alpha)p_{\theta^{(t)}} + \alpha p_{\mathrm{sft}}.
\label{eq:linear_recursion}
\end{equation}

Unrolling Eq.~\eqref{eq:linear_recursion} gives
\begin{align}
p_{\theta^{(t)}}
&= (1-\alpha)^t p_{\theta^{(0)}} 
+ \sum_{k=0}^{t-1} (1-\alpha)^k \alpha p_{\mathrm{sft}} \notag \\
&= (1-\alpha)^t p_{\mathrm{base}}
+ \big(1-(1-\alpha)^t\big)p_{\mathrm{sft}},
\label{eq:closed_form_dynamic}
\end{align}
where we used the initialization $p_{\theta^{(0)}} = p_{\mathrm{base}}$ and the
finite geometric series identity.
Since $(1-\alpha)^t \to 0$ as $t \to \infty$, the sequence converges pointwise to
$p_{\mathrm{sft}}$.
Moreover, the convergence rate is geometric with ratio $(1-\alpha)$.

To quantify the convergence in divergence, we invoke the joint convexity of KL
divergence in its first argument:
\begin{align}
\mathrm{KL}\big(p_{\theta^{(t)}} \,\|\, p_{\mathrm{sft}}\big)
&=
\mathrm{KL}\big((1-\alpha)^t p_{\mathrm{base}}
+ (1-(1-\alpha)^t)p_{\mathrm{sft}}
\,\|\, p_{\mathrm{sft}}\big) \notag \\
&\le (1-\alpha)^t
\mathrm{KL}\big(p_{\mathrm{base}} \,\|\, p_{\mathrm{sft}}\big).
\label{eq:dynamic_kl_decay}
\end{align}
Hence, the divergence to the SFT distribution decays exponentially over outer
iterations.

\paragraph{Static KL Barycenter.}
Consider the static optimization problem
\begin{equation}
p^{\star}
=
\argmin_{p \in \Delta}
\;
(1-\beta)\,\mathrm{KL}\big(p \,\|\, p_{\mathrm{sft}}\big)
+
\beta\,\mathrm{KL}\big(p \,\|\, p_{\mathrm{base}}\big),
\qquad \beta \in (0,1).
\label{eq:static_problem_app}
\end{equation}
By an argument analogous to the derivation in this appendix (Proof of
Lemma~\ref{lem:barycenter}), the unique minimizer admits the closed form
\begin{equation}
p^{\star}_i
=
\frac{1}{Z_{\beta}}
\, p_{\mathrm{sft},i}^{\,1-\beta}
\, p_{\mathrm{base},i}^{\,\beta},
\qquad
Z_{\beta} =
\sum_j p_{\mathrm{sft},j}^{\,1-\beta} p_{\mathrm{base},j}^{\,\beta}.
\label{eq:static_solution_app}
\end{equation}
Unless $\beta = 0$, $p^{\star}$ differs from $p_{\mathrm{sft}}$ and therefore
remains at a strictly positive KL divergence:
\[
\mathrm{KL}(p^{\star} \,\|\, p_{\mathrm{sft}}) > 0.
\]
Consequently, static regularization enforces convergence to a fixed compromise
distribution rather than to the task-optimal distribution.

\paragraph{Comparison.}
Equations~\eqref{eq:closed_form_dynamic}--\eqref{eq:dynamic_kl_decay} show that the
dynamic anchor recursion converges geometrically to $p_{\mathrm{sft}}$ while
maintaining controlled stepwise deviation from the current model.
In contrast, Eq.~\eqref{eq:static_solution_app} demonstrates that static KL
regularization converges in a single step to a fixed barycenter that does not
evolve with the optimization trajectory.
This formalizes the fundamental distinction between progressive adaptation and
static compromise.

\section{Detailed Experimental Results}
\label{sec:detailed}

To supplement the visual results discussed in the main text, we report the comprehensive numerical data in this section. Tables \ref{tab:performance_comparison}, \ref{tab:medcalc_performance}, and \ref{tab:ifeval_performance} provide the exact accuracy and pass rates for the comparative experiments on iGSM, MedCalc, and IFEval, respectively.

\begin{table*}[h]
\centering
\small
\caption{Detailed numerical results based on Qwen2.5-3B-Instruct across general and iGSM benchmark.}
\label{tab:performance_comparison}
% 使用 resizebox 自动调整表格宽度以适应页面
\resizebox{\textwidth}{!}{%
\begin{tabular}{lcccccccc}
\toprule
\multirow{4}{*}{\textbf{Method}} & \multicolumn{6}{c}{\textbf{General}} & \multicolumn{1}{c}{\textbf{Domain}} & \multirow{4}{*}{\textbf{Overall Avg.}} \\
\cmidrule(lr){2-7} \cmidrule(lr){8-8}
 & \multicolumn{1}{c}{Science} & \multicolumn{4}{c}{Coding} & \multicolumn{1}{c}{Math} & \multicolumn{1}{c}{Math} & \\
\cmidrule(lr){2-2} \cmidrule(lr){3-6} \cmidrule(lr){7-7} \cmidrule(lr){8-8}
 & mmlu pro & humaneval & humaneval+ & mbpp & mbpp+ & countdown & iGSM & \\
 & acc & pass@1 & pass@1 & pass@1 & pass@1 & acc & acc & \\
\midrule
Base & 0.419 & 0.738 & 0.677 & 0.743 & 0.622 & 0.225 & 0.184 & 0.515 \\
SFT & 0.031 & 0.012 & 0.012 & 0.026 & 0.019 & 0.000 & 1.000 & 0.157 \\
Low-SFT & 0.333 & 0.689 & 0.646 & 0.714 & 0.590 & 0.004 & 0.930 & 0.558 \\
% KL-SFT 0.05 & 0.062 & 0.030 & 0.030 & 0.101 & 0.093 & 0.000 & 0.886 & 0.172 \\
KL-SFT & 0.102 & 0.140 & 0.128 & 0.185 & 0.159 & 0.000 & 0.606 & 0.189 \\
% KL-SFT 0.5 & 0.127 & 0.165 & 0.128 & 0.241 & 0.193 & 0.000 & 0.352 & 0.172 \\
STM & 0.321 & 0.713 & 0.646 & 0.693 & 0.598 & 0.078 & 0.594 & 0.520 \\
DFT & 0.163 & 0.457 & 0.421 & 0.481 & 0.423 & 0.000 & 0.524 & 0.353 \\
Self-SFT & 0.386 & 0.713 & 0.665 & 0.754 & 0.656 & 0.250 & 0.250 & 0.525 \\
Iter-SFT & 0.377 & 0.713 & 0.646 & 0.741 & 0.638 & 0.290 & 0.294 & 0.528 \\
Ours & 0.377 & 0.701 & 0.640 & 0.704 & 0.593 & 0.184 & 0.936 & \textbf{0.591} \\
\bottomrule
\end{tabular}%
}
\end{table*}

\begin{table*}[h]
\centering
\small
\caption{Detailed numerical results based on Qwen2.5-3B-Instruct across general and MedCalc benchmark.}
\label{tab:medcalc_performance}
\resizebox{\textwidth}{!}{%
\begin{tabular}{lcccccccc}
\toprule
\multirow{4}{*}{\textbf{Method}} & \multicolumn{6}{c}{\textbf{General}} & \multicolumn{1}{c}{\textbf{Domain}} & \multirow{4}{*}{\textbf{Overall Avg.}} \\
\cmidrule(lr){2-7} \cmidrule(lr){8-8}
 & \multicolumn{1}{c}{Science} & \multicolumn{4}{c}{Coding} & \multicolumn{1}{c}{Math} & \multicolumn{1}{c}{Medical} & \\
\cmidrule(lr){2-2} \cmidrule(lr){3-6} \cmidrule(lr){7-7} \cmidrule(lr){8-8}
 & mmlu pro & humaneval & humaneval+ & mbpp & mbpp+ & countdown & MedCalc & \\
 & acc & pass@1 & pass@1 & pass@1 & pass@1 & acc & acc & \\
\midrule
Base & 0.419 & 0.738 & 0.677 & 0.743 & 0.622 & 0.225 & 0.133 & 0.508 \\
SFT & 0.109 & 0.012 & 0.012 & 0.012 & 0.012 & 0.000 & 0.574 & 0.104 \\
Low-SFT & 0.317 & 0.665 & 0.604 & 0.680 & 0.556 & 0.047 & 0.544 & 0.488 \\
% KL-SFT 0.05 & 0.086 & 0.122 & 0.104 & 0.095 & 0.079 & 0.000 & 0.502 & 0.141 \\
KL-SFT & 0.096 & 0.104 & 0.079 & 0.143 & 0.127 & 0.000 & 0.356 & 0.129 \\
% KL-SFT 0.5 & 0.118 & 0.183 & 0.165 & 0.243 & 0.217 & 0.000 & 0.198 & 0.161 \\
STM & 0.277 & 0.646 & 0.598 & 0.693 & 0.569 & 0.046 & 0.147 & 0.425 \\
DFT & 0.297 & 0.604 & 0.530 & 0.635 & 0.550 & 0.043 & 0.145 & 0.401 \\
Self-sft & 0.394 & 0.689 & 0.671 & 0.661 & 0.556 & 0.191 & 0.216 & 0.483 \\
Iter-SFT & 0.338 & 0.640 & 0.561 & 0.733 & 0.611 & 0.209 & 0.290 & 0.483 \\
Ours & 0.356 & 0.683 & 0.628 & 0.704 & 0.598 & 0.096 & 0.563 & \textbf{0.518} \\
\bottomrule
\end{tabular}%
}
\end{table*}

\begin{table*}[h]
\centering
\small
\caption{Detailed numerical results based on Qwen2.5-3B-Instruct across general and IFEval benchmark.}
\label{tab:ifeval_performance}
\resizebox{\textwidth}{!}{%
\begin{tabular}{lcccccccc}
\toprule
\multirow{4}{*}{\textbf{Method}} & \multicolumn{6}{c}{\textbf{General}} & \multicolumn{1}{c}{\textbf{Domain}} & \multirow{4}{*}{\textbf{Overall Avg.}} \\
\cmidrule(lr){2-7} \cmidrule(lr){8-8}
 & \multicolumn{1}{c}{Science} & \multicolumn{4}{c}{Coding} & \multicolumn{1}{c}{Math} & \multicolumn{1}{c}{Ins Following} & \\
\cmidrule(lr){2-2} \cmidrule(lr){3-6} \cmidrule(lr){7-7} \cmidrule(lr){8-8}
 & mmlu pro & humaneval & humaneval+ & mbpp & mbpp+ & countdown & ifeval & \\
 & acc & pass@1 & pass@1 & pass@1 & pass@1 & acc & acc & \\
\midrule
Base & 0.419 & 0.738 & 0.677 & 0.743 & 0.622 & 0.225 & 0.398 & 0.546 \\
SFT & 0.375 & 0.683 & 0.622 & 0.696 & 0.582 & 0.138 & 0.694 & 0.541 \\
Low-SFT & 0.417 & 0.726 & 0.659 & 0.741 & 0.611 & 0.171 & 0.449 & 0.539 \\
KL-SFT& 0.383 & 0.744 & 0.671 & 0.680 & 0.595 & 0.056 & 0.504 & 0.519 \\
STM & 0.361 & 0.732 & 0.671 & 0.709 & 0.598 & 0.121 & 0.585 & 0.540 \\
DFT & 0.413 & 0.701 & 0.659 & 0.722 & 0.630 & 0.168 & 0.575 & 0.553 \\
Self-SFT & 0.411 & 0.671 & 0.640 & 0.701 & 0.595 & 0.237 & 0.526 & 0.540 \\
Iter-SFT & 0.358 & 0.701 & 0.659 & 0.717 & 0.595 & 0.209 & 0.526 & 0.538 \\
Ours & 0.409 & 0.707 & 0.646 & 0.706 & 0.611 & 0.170 & 0.652 & \textbf{0.557} \\
\bottomrule
\end{tabular}%
}
\end{table*}

\begin{table*}[h]
\centering

\caption{Detailed numerical results based on Llama-3.2-3B-Instruct across general and iGSM benchmark.}
\label{tab:performance_comparison_new}
\resizebox{\textwidth}{!}{%
\begin{tabular}{lcccccccc}
\toprule
\multirow{4}{*}{\textbf{Method}} & \multicolumn{6}{c}{\textbf{General}} & \multicolumn{1}{c}{\textbf{Domain}} & \multirow{4}{*}{\textbf{Overall Avg.}} \\
\cmidrule(lr){2-7} \cmidrule(lr){8-8}
 & \multicolumn{1}{c}{Science} & \multicolumn{4}{c}{Coding} & \multicolumn{1}{c}{Math} & \multicolumn{1}{c}{Math}& \\
\cmidrule(lr){2-2} \cmidrule(lr){3-6} \cmidrule(lr){7-7} \cmidrule(lr){8-8}
 & mmlu pro & humaneval & humaneval+ & mbpp & mbpp+ & countdown & iGSM & \\
 & acc & pass@1 & pass@1 & pass@1 & pass@1 & acc & acc & \\
\midrule
Base & 0.347 & 0.549 & 0.500 & 0.627 & 0.516 & 0.051 & 0.166 & 0.394 \\
SFT & 0.012 & 0.060 & 0.060 & 0.012 & 0.012 & 0.000 & 1.000 & 0.191 \\
Low-SFT & 0.363 & 0.530 & 0.470 & 0.638 & 0.540 & 0.000 & 0.648 & 0.456 \\
KL-SFT & 0.020 & 0.183 & 0.165 & 0.243 & 0.217 & 0.000 & 0.624 & 0.207 \\
STM & 0.292 & 0.506 & 0.457 & 0.643 & 0.534 & 0.000 & 0.492 & 0.418 \\
DFT & 0.082 & 0.232 & 0.207 & 0.291 & 0.220 & 0.000 & 0.660 & 0.242 \\
Self-SFT & 0.312 & 0.543 & 0.476 & 0.534 & 0.466 & 0.070 & 0.242 & 0.378 \\
Iter-SFT & 0.322 & 0.564 & 0.503 & 0.564 & 0.494 & 0.050 & 0.272 & 0.396 \\
Ours & 0.343 & 0.573 & 0.512 & 0.669 & 0.548 & 0.000 & 0.994 & \textbf{0.520} \\
\bottomrule
\end{tabular}%
}
\end{table*}

\begin{table*}[h]
\centering
% 请根据实际使用的模型修改 [Model Name]，例如 Qwen2.5-1.5B-Instruct
\caption{Detailed numerical results based on Llama-3.2-3B-Instruct across general and MedCalc benchmark.}
\label{tab:medcalc_detailed}
\resizebox{\textwidth}{!}{%
\begin{tabular}{lcccccccc}
\toprule
\multirow{4}{*}{\textbf{Method}} & \multicolumn{6}{c}{\textbf{General}} & \multicolumn{1}{c}{\textbf{Domain}} & \multirow{4}{*}{\textbf{Overall Avg.}} \\
\cmidrule(lr){2-7} \cmidrule(lr){8-8}
 & \multicolumn{1}{c}{Science} & \multicolumn{4}{c}{Coding} & \multicolumn{1}{c}{Math} &  \multicolumn{1}{c}{Medical} & \\
\cmidrule(lr){2-2} \cmidrule(lr){3-6} \cmidrule(lr){7-7} \cmidrule(lr){8-8}
 & mmlu pro & humaneval & humaneval+ & mbpp & mbpp+ & countdown & MedCalc & \\
 & acc & pass@1 & pass@1 & pass@1 & pass@1 & acc & acc & \\
\midrule
Base & 0.347 & 0.549 & 0.500 & 0.627 & 0.516 & 0.051 & 0.078 & 0.381 \\
SFT & 0.011 & 0.060 & 0.060 & 0.019 & 0.019 & 0.000 & 0.547 & 0.102 \\
Low-SFT & 0.332 & 0.537 & 0.488 & 0.611 & 0.500 & 0.012 & 0.513 & 0.428 \\
KL-SFT 0.2 & 0.096 & 0.104 & 0.079 & 0.101 & 0.093 & 0.000 & 0.356 & 0.118 \\
STM & 0.243 & 0.537 & 0.470 & 0.579 & 0.484 & 0.006 & 0.085 & 0.343 \\
DFT & 0.280 & 0.457 & 0.415 & 0.574 & 0.481 & 0.000 & 0.145 & 0.336 \\
Self-SFT & 0.269 & 0.537 & 0.476 & 0.601 & 0.500 & 0.016 & 0.173 & 0.367 \\
Iter-SFT & 0.280 & 0.577 & 0.509 & 0.619 & 0.535 & 0.012 & 0.238 & 0.396 \\
Ours & 0.333 & 0.561 & 0.500 & 0.677 & 0.455 & 0.019 & 0.564 & \textbf{0.444} \\
\bottomrule
\end{tabular}%
}
\end{table*}

\begin{table*}[h]
\centering
\caption{Detailed numerical results based on Llama-3.2-3B-Instruct across general and IFEval benchmark.}
\label{tab:ifeval_detailed_comparison}
\resizebox{\textwidth}{!}{%
\begin{tabular}{lcccccccc}
\toprule
\multirow{4}{*}{\textbf{Method}} & \multicolumn{6}{c}{\textbf{General}} & \multicolumn{1}{c}{\textbf{Domain}} & \multirow{4}{*}{\textbf{Overall Avg.}} \\
\cmidrule(lr){2-7} \cmidrule(lr){8-8}
 & \multicolumn{1}{c}{Science} & \multicolumn{4}{c}{Coding} & \multicolumn{1}{c}{Math} &  \multicolumn{1}{c}{Ins Following} & \\
\cmidrule(lr){2-2} \cmidrule(lr){3-6} \cmidrule(lr){7-7} \cmidrule(lr){8-8}
 & mmlu pro & humaneval & humaneval+ & mbpp & mbpp+ & countdown & ifeval & \\
 & acc & acc & acc & acc & acc & acc & acc & \\
\midrule
Base & 0.347 & 0.549 & 0.500 & 0.627 & 0.516 & 0.051 & 0.356 & 0.421 \\
SFT & 0.327 & 0.507 & 0.425 & 0.520 & 0.420 & 0.012 & 0.656 & 0.410 \\
Low-SFT & 0.378 & 0.524 & 0.457 & 0.558 & 0.466 & 0.035 & 0.642 & 0.437 \\
KL-SFT & 0.355 & 0.494 & 0.410 & 0.544 & 0.466 & 0.013 & 0.428 & 0.387  \\
STM & 0.333 & 0.573 & 0.494 & 0.696 & 0.590 & 0.052 & 0.568 & 0.472 \\
DFT & 0.347 & 0.543 & 0.488 & 0.622 & 0.534 & 0.046 & 0.570 & 0.450 \\
Self-SFT & 0.296 & 0.543 & 0.476 & 0.534 & 0.466 & 0.057 & 0.496 & 0.410 \\
Iter-SFT & 0.309 & 0.571 & 0.515 & 0.569 & 0.497 & 0.050 & 0.526 & 0.434 \\
Ours & 0.368 & 0.573 & 0.518 & 0.587 & 0.481 & 0.050 & 0.600 & \textbf{0.454} \\
\bottomrule
\end{tabular}%
}
\end{table*}

\section{Evaluation across the Incremental Learning Stages}
Tab.~\ref{tab:continual_learning} reports performance across three incremental adaptation stages on Qwen2.5-3B-Instruct.
At Stage~1 (iGSM), both methods achieve strong domain adaptation, while Anchored Learning preserves higher general performance than Low-SFT (0.533 vs.\ 0.496 in General Avg) and attains slightly better overall performance.
At Stage~2 (iGSM$\rightarrow$MedCalc), the gap in domain performance becomes more pronounced: Anchored Learning substantially improves medical accuracy (0.512 vs.\ 0.505) and yields a higher Domain Avg (0.577 vs.\ 0.470), while maintaining comparable general performance.
At Stage~3 (iGSM$\rightarrow$MedCalc$\rightarrow$IFEval), Anchored Learning consistently dominates Low-SFT across both domain and general metrics, achieving a markedly higher Domain Avg (0.505 vs.\ 0.331) and the best Overall Avg (0.492 vs.\ 0.459).
These results demonstrate that Anchored Learning accumulates substantially less performance degradation across successive adaptation stages, highlighting its effectiveness in mitigating compounding distributional drift under multi-stage fine-tuning.

\begin{table*}[t]
\centering
\caption{Performance comparisons across incremental learning stages (iGSM → MedCalc → IFEval) on Qwen2.5-3B-Instruct.}
\label{tab:continual_learning}
% 调整列间距，使表格更紧凑
\setlength{\tabcolsep}{3.5pt}
\resizebox{\textwidth}{!}{%
\begin{tabular}{lcccccccccccc}
\toprule
\multirow{4}{*}{\textbf{Method}} & \multicolumn{7}{c}{\textbf{General}} & \multicolumn{4}{c}{\textbf{Domain}} & \multirow{4}{*}{\shortstack{\textbf{Overall Avg.}}} \\
\cmidrule(lr){2-8} \cmidrule(lr){9-12}
 & Science & \multicolumn{4}{c}{Coding} & Math & \multirow{3}{*}{\shortstack{Gen. \\ Avg.}} & Math & Medical & Ins. Foll. & \multirow{3}{*}{\shortstack{Dom. \\ Avg.}} & \\
\cmidrule(lr){2-2} \cmidrule(lr){3-6} \cmidrule(lr){7-7} \cmidrule(lr){9-9} \cmidrule(lr){10-10} \cmidrule(lr){11-11}
 & mmlu pro & hum & hum+ & mbpp & mbpp+ & count & & iGSM & medcalc & ifeval & & \\
 & acc & pass@1 & pass@1 & pass@1 & pass@1 & acc & & acc & acc & acc & & \\
\midrule
Base & 0.419 & 0.738 & 0.677 & 0.743 & 0.622 & 0.225 & 0.571 & 0.184 & 0.133 & 0.398 & 0.238 & 0.466 \\
\midrule
\multicolumn{13}{c}{\textit{\textbf{Stage 1: iGSM}}} \\
\midrule
Low-SFT & 0.333 & 0.689 & 0.646 & 0.714 & 0.590 & 0.004 & 0.496 & 0.930 & 0.100 & 0.292 & 0.441 & 0.423 \\
Ours & 0.377 & 0.701 & 0.640 & 0.704 & 0.593 & 0.184 & 0.533 & 0.936 & 0.073 & 0.367 & 0.459 & 0.455 \\
\midrule
\multicolumn{13}{c}{\textit{\textbf{Stage 2: iGSM $\to$ MedCalc}}} \\
\midrule
Low-SFT & 0.299 & 0.646 & 0.591 & 0.677 & 0.553 & 0.003 & 0.462 & 0.622 & 0.505 & 0.284 & 0.470 & 0.448 \\
Ours & 0.308 & 0.610 & 0.579 & 0.643 & 0.545 & 0.002 & 0.448 & 0.884 & 0.512 & 0.334 & 0.577 & 0.457 \\
\midrule
\multicolumn{13}{c}{\textit{\textbf{Stage 3: iGSM $\to$ MedCalc $\to$ IFEval}}} \\
\midrule
Low-SFT & 0.368 & 0.713 & 0.652 & 0.659 & 0.577 & 0.170 & 0.523 & 0.330 & 0.157 & 0.506 & 0.331 & 0.459 \\
Ours & 0.366 & 0.707 & 0.665 & 0.704 & 0.603 & 0.061 & 0.518 & 0.700 & 0.297 & 0.519 & 0.505 & \textbf{0.492} \\
\bottomrule
\end{tabular}%
}
\end{table*}

\section{Algorithmic Details of Anchored Learning}
This section provides the full pseudocode of Anchored Learning for completeness.

Alg.~\ref{alg:anchored_learning} summarizes the iterative procedure. The algorithm takes as input a base model $p_{\mathrm{base}}$, a fixed reference model $p_{\mathrm{sft}}$, an interpolation coefficient $\alpha \in (0,1)$, and the total number of outer iterations $T$. The model is initialized from the base distribution and updated over $T$ outer iterations.

At each outer iteration $t$, an anchor distribution $q^{(t)}(\cdot \mid x)$ is first constructed by interpolating between the current model $p_{\theta^{(t)}}$ and the fixed reference model according to Eq.~(2). The model parameters are then updated by distilling toward the anchor distribution via the objective in Eq.~(3), yielding the next iterate $p_{\theta^{(t+1)}}$. The final adapted model after $T$ iterations is returned as $p_{\theta^{(T)}}$.

\begin{algorithm}[t]
\caption{Anchored Learning}
\label{alg:anchored_learning}
\begin{algorithmic}[1]
\REQUIRE Base model $p_{\mathrm{base}}$, fixed SFT model $p_{\mathrm{sft}}$, interpolation coefficient $\alpha \in (0,1)$, outer iterations $T$
\ENSURE Adapted model $p_{\theta^{(T)}}$

\STATE Initialize $p_{\theta^{(0)}} \gets p_{\mathrm{base}}$

\FOR{$t = 0,1,\ldots,T-1$}
    \STATE \textbf{Step 1: Anchor Selection via Interpolation}
    \STATE Compute $q^{(t)}(\cdot\mid x)$ according to Eq.~\ref{eq:anchor_general}

 \STATE \textbf{Step 2: Anchor Fitting via Distillation}
    \STATE Update $\theta^{(t+1)}$ according to Eq.~\ref{eq:inner_argmin}
    
\ENDFOR

\STATE \textbf{return} $p_{\theta^{(T)}}$
\end{algorithmic}
\end{algorithm}

\section{Case Study}
\label{app:case_study}

To demonstrate the effectiveness of our method in preventing catastrophic forgetting, we present a case study involving a general arithmetic reasoning task ("Make 78"). The models were fine-tuned on the \textbf{iGSM dataset}, a domain-specific corpus focused on solving linear equations involving variables (e.g., expressions similar to $x + y + z$).

As shown in Figure~\ref{fig:math_case_study}, the \textbf{Base Model} successfully leverages its pre-trained general reasoning capabilities to deduce the correct arithmetic combination ($77 + 56 - 55 = 78$).

The behavior of the \textbf{Low-SFT} model highlights the impact of \textit{negative transfer} from the specific domain. Although the model remains compliant with the output format instructions (producing a boxed answer), its internal reasoning is heavily biased by the SFT data distribution. Instead of performing the necessary arithmetic search, the model mimics the syntactic patterns of the iGSM dataset—defining variables and constructing rigid linear equations—which leads to a logical failure on this general reasoning task.

In contrast, \textbf{Anchored Learning} effectively mitigates this interference. It resists the tendency to overfit to the specific stylistic patterns of the iGSM data, thereby retaining the robust reasoning capabilities of the Base Model and correctly solving the problem.

\begin{figure*}[h!]
    \centering
    \begin{tcolorbox}[
        colback=gray!5,
        colframe=gray!50,
        title=\textbf{Input Prompt},
        fonttitle=\bfseries,
        sharp corners=south,
        boxrule=0.5mm
    ]
    In this task, you need to use a list of numbers x = (55, 56, 77) to create an equation that leads up to the target number y = 78 using the basic arithmetic operations (+, -, *, /), and each number can only be used once. Think and return the final answer in $\boxed{ }$. For example, with x = (1, 2, 3) and y = 1, your answer should be $\boxed{(1 + 2) / 3}$.
    \end{tcolorbox}
    
    \vspace{-0.3cm}
    
    % Base Model Response
    \begin{tcolorbox}[
        colback=white,
        colframe=blue!40!black,
        title=\textbf{Base Model Response (Pre-trained)},
        fonttitle=\bfseries\small,
        sharp corners,
        boxrule=0.3mm
    ]
    \small
    To solve the problem for the given target values $x = (55, 56, 77)$ and $y = 78$, we need to construct an equation using the numbers 55, 56, and 77 exactly once...
    
    [...Step-by-step derivation omitted for brevity...]
    
    Consider the equation $77 + 56 - 55$:
    First, calculate $77 + 56 = 133$. Then, subtract 55 from 133: $133 - 55 = 78$.
    
    Therefore, the final answer is:
    $\boxed{77 + 56 - 55}$ \textcolor{green!60!black}{\textbf{\checkmark (Correct)}}
    \end{tcolorbox}
    
    \vspace{-0.3cm}

    % Baseline Model Response
    \begin{tcolorbox}[
        colback=red!5,
        colframe=red!60!black,
        title=\textbf{Baseline Model Response (Low-SFT)},
        fonttitle=\bfseries\small,
        sharp corners,
        boxrule=0.3mm
    ]
    \small
    Define x as a list, x = [55, 56, 77]; Define y as 78; Define some variables as V; Define the equation as: y = (x[0] + x[1]) * x[2] = 78; The final answer is \boxed{y}. 
    
    \vspace{0.2cm}
    \textit{\textbf{Analysis:} The model fails to solve the math problem. It hallucinates a coding structure and incorrect logic ($(55+56)*77 \neq 78$), indicating catastrophic forgetting of general arithmetic reasoning.} \textcolor{red!60!black}{\textbf{\xmark (Incorrect)}}
    \end{tcolorbox}
    
    \vspace{-0.3cm}

    % Our Model Response
    \begin{tcolorbox}[
        colback=green!5,
        colframe=green!50!black,
        title=\textbf{Our Model Response (Anchored Learning)},
        fonttitle=\bfseries\small,
        sharp corners=north,
        boxrule=0.3mm
    ]
    \small
    To solve for the target number 78 using the numbers 55, 56, and 77 with the basic arithmetic operations, we need to strategically apply the operations. Let's break it down step by step:
    
    1. Start with the given numbers: 55, 56, and 77.
    2. Notice that 77 is close to 78, so we can use 77 directly and adjust it to 78 by adding 1.
    3. To get 1, we can use $56 - 55 = 1$.
    4. So, the equation is $77 + (56 - 55) = 78$.
    
    Therefore, the final answer is $\boxed{77 + 56 - 55}$. \textcolor{green!60!black}{\textbf{\checkmark (Correct)}}
    \end{tcolorbox}
    
    \caption{\textbf{Case Study on Countdown Task.} 
    % Comparison of model outputs on an arithmetic reasoning task. 
    While the Baseline model loses its ability to perform basic math (likely due to overfitting on the domain-specific fine-tuning data), our method preserves the reasoning capabilities inherent in the Base Model.}
    \label{fig:math_case_study}
\end{figure*}

%%%%%%%%%%%%%%%%%%%%%%%%%%%%%%%%%%%%%%%%%%%%%%%%%%%%%%%%%%%%%%%%%%%%%%%%%%%%%%%
%%%%%%%%%%%%%%%%%%%%%%%%%%%%%%%%%%%%%%%%%%%%%%%%%%%%%%%%%%%%%%%%%%%%%%%%%%%%%%%

\end{document}